\documentclass[10pt]{article} % For LaTeX2e
\usepackage[accepted]{tmlr}
\usepackage[T1]{fontenc}
\usepackage{graphicx}
\usepackage[utf8]{inputenc}
\usepackage[american]{babel}
\usepackage{amssymb, amsmath, amsthm}
\usepackage{mathtools}
\usepackage{paralist}
\usepackage{booktabs}
\usepackage{xcolor}
\usepackage{diagbox}
\usepackage{bbm}
\usepackage{paralist}
\usepackage[algoruled,vlined,linesnumbered]{algorithm2e}
\usepackage[colorlinks,citecolor=blue,urlcolor=black,hidelinks,linktocpage]{hyperref}
\usepackage{cleveref}
\usepackage{etoolbox}

\DeclareMathOperator{\partdiam}{PartDiam}
\DeclareMathOperator{\obsdiam}{ObsDiam}
\DeclareMathOperator{\diam}{diam}
\DeclareMathOperator{\E}{E}
\DeclareMathOperator{\costs}{C}
\DeclareMathOperator{\scosts}{SC}
\DeclareMathOperator{\mle}{MLE}
\usepackage[textsize=scriptsize, textwidth=14ex]{todonotes}
\newtheorem{theorem}{Theorem}[section]
\newtheorem{corollary}[theorem]{Corollary}
\newtheorem{lemma}[theorem]{Lemma}
\newtheorem{definition}[theorem]{Definition}

\let\cref\Cref

\usepackage{hyperref}
\usepackage{url}

\title{Intrinsic Dimension for Large-Scale Geometric Learning}

\author{Maximilian Stubbemann \email stubbemann@cs.uni-kassel.de \\
      \addr Knowledge \& Data Engineering Group, University of Kassel, Kassel, Germany
      \AND
      \name Tom Hanika \email tom.hanika@cs.uni-kassel.de \\
      \addr Knowledge \& Data Engineering Group, University of Kassel, Kassel, Germany
      \AND
      \name Friedrich Martin Schneider \email martin.schneider@math.tu-freiberg.de\\
      \addr Institute of Discrete Mathematics and Algebra, TU Bergakademie Freiberg, Freiberg, Germany}

\def\month{01}  % Insert correct month for camera-ready version
\def\year{2023} % Insert correct year for camera-ready version
\def\openreview{\url{https://openreview.net/forum?id=85BfDdYMBY}} % Insert correct link to OpenReview for camera-ready version

\begin{document}

\maketitle

\begin{abstract}
  The concept of dimension is essential to grasp the complexity of
  data. A naive approach to determine the dimension of a dataset is
  based on the number of attributes. More sophisticated methods derive
  a notion of intrinsic dimension (ID) that employs more complex
  feature functions, e.g., distances between data points.  Yet, many
  of these approaches are based on empirical observations, cannot cope
  with the geometric character of contemporary datasets, and do lack
  an axiomatic foundation. A different approach was proposed by
  V. Pestov, who links the intrinsic dimension axiomatically to the
  mathematical concentration of measure phenomenon. First methods to
  compute this and related notions for ID were computationally
  intractable for large-scale real-world datasets. In the present
  work, we derive a computationally feasible method for determining
  said axiomatic ID functions. Moreover, we demonstrate how the
  geometric properties of complex data are accounted for in our
  modeling.  In particular, we propose a principle way to incorporate
  neighborhood information, as in graph data, into the ID. This allows
  for new insights into common graph learning procedures, which we
  illustrate by experiments on the Open Graph Benchmark.
\end{abstract}

\section{Introduction}
Contemporary real-world datasets employed in
artificial intelligence are often large in size and comprised of complex structures, which
distinguishes them from Euclidean data. To consider these properties appropriately is a challenging task
for procedures that analyze or learn from said data. Moreover, with
increasing complexity of real-world data, the necessity arises to
quantify to which extent this data suffer from the curse of dimensionality.
The common approach for estimating the dimension curse of a particular
dataset is through the notion of \emph{intrinsic dimension} (ID)
\citep{bac20, granata16, pestov07}. There exists a variety of work on
how to estimate the ID of
datasets~\citep{facco17,levina04,costa05,gomtsyan19,bac20}. Most
approaches to quantify the ID are based on distances between data
points, assuming the data to be Euclidean. A multitude of works base
their modeling on the manifold
hypothesis~\citep{cloninger21,gomtsyan19}, which assumes that the
observed data is embedded in a manifold of low dimension (compared to
the number of data attributes). The ID then is an approximation of the
dimension of this manifold. \citet{pestov00} proposed a different
concept of intrinsic dimension by linking it to the mathematical
concentration of measure phenomenon. His modeling is based on a
thorough axiomatic approach~\citep{pestov07,pestov08,pestov10} which
resulted in a novel class of intrinsic dimension functions. In
contrast to the manifold hypothesis, Pestov's ID functions measure to
which extent a dataset is affected by the curse of dimensionality,
i.e., to which extent the complexity of the dataset hinders the
discrimination of data points. Yet, to compute said ID functions is an
intractable computational endeavor. This limitation was overcome in
principle by an adaptation to \emph{geometric
  datasets}~\citep{hanika22}. However, two limitations persisted:
First, the computational effort was found to remain quadratic in the
number of data points, which is insufficient for datasets of
contemporary size; second, it is unclear how to account for complex
structure, such as in graph data.

With this in mind, we propose in the present work a default approach
for computing the intrinsic dimension of geometric data, such as graph
data, as used in graph neural networks. To do this, we revisit the
computation of the ID based on distance functions~\citep{hanika22} and
overcome, in particular, the inherent computational limitations in the
works by~\cite{pestov07} and~\cite{hanika22}. In detail, we derive a novel
approximation formula and present an algorithm for its
computation. This allows us to compute ID bounds for datasets that are
magnitudes larger than in earlier works. That equipped with, we establish a natural
approach to compute the ID of graph data.

We subsequently apply our method to seven real-world datasets and
relate the obtained results to the observed performances of
classification procedures. Thus, we demonstrate the practical
computability of our approach. In addition, we study the extent to
which the intrinsic dimension reveals insights into the performance of particularly
classes of Graph Neural Networks. Our code is publicly available on GitHub.\footnote{\url{https://github.com/mstubbemann/ID4GeoL}}

\section{Related Work}
\label{sec:related-work}
In numerous works, the intrinsic dimension is estimated using
the pairwise distances between data points~\citep{chavez01,
  grassberger04}.  More sophisticated approaches use distances to
nearest neighbors~\citep{facco17, levina04, costa05, gomtsyan19}. All
these works have in common, that they assume the data to be
Euclidean and that they favor local properties.

Recent work has drawn different connections between intrinsic
dimension (ID) and modern learning theory. For example, \citet{cloninger21} show that functions of the form $f(x) = g (\phi(x))$,
where $\phi$ maps into a manifold of lower dimension, can be
approximated by neural networks. On the other hand, \citet{wojtowytsch20} prove that modern artificial neural
networks suffer from the curse of dimensionality in the sense that
gradient training on high dimensional data may converge
insufficiently. Additional to these theoretical results, there is an
increasing interest of empirically estimating the ID of contemporary
learning architectures. \citet{chunyuan18} study the ID
of neural networks by replacing high dimensional parameter vectors
with lower dimensional ones. Their approach results in a
non-deterministic ID. More recent works studied ID in the realm of
geometric data and their standard architectures. \citet{ansuini19} investigate the ID for convolutional neural
networks (CNN). In detail, they are interested to which extent the ID
changes at different hidden layers and how this is related to the
overall classification performance. Another work~\citep{pope20} associates an ID to popular benchmark image
datasets. These two works on ID estimators do solely rely on the
metric information of the data and do not consider any geometric
structure of image data.

Our approach allows to incorporate such underlying geometric
structures while incorporating the mathematical phenomenon of measure
concentration~\citep{gromov83, milman88, milman10}. Linking this
phenomenon to the occurrence of the dimension curse was done by
Pestov~\citep{pestov00, pestov07, pestov08, pestov10}. He based his
considerations on a thorough axiomatic approach using techniques from
metric-measure spaces. The resulting ID functions unfortunately turn
out to be practically incomputable. In contrast, \citet{bac20} investigate computationally feasible ID
estimators that are related to the concentration phenomenon. Yet,
their results elude a comparable axiomatic foundation.  Our modeling for
the ID of large and geometric data is based on~\citet{hanika22}. We build on their axiomatization and derive a
computationally feasible method for the intrinsic dimension of
large-scale geometric datasets.

\section{Intrinsic Dimension}
\label{sec:id}
Since our work is based on the formalization from~\citet{hanika22}, we
shortly revisit their modeling and recapitulate the most important
structures. Based on this, we derive and prove an explicit formula to
compute the ID for the special case of finite geometric datasets. This
first result is essential for \cref{sec:approx}.

Let $\mathcal{D}=(X,F,\mu)$, where $X$ is a set and $F \subseteq
{\mathbb{R}}^{X}$ is a set of functions from $X$ to $\mathbb{R}$,
in the following called \emph{feature functions}. We require that
$\sup_{x,y\in X}d_F(x,y) < \infty$, where $d_F(x,y) \coloneqq \sup_{f
  \in F}|f(x) - f(y)|$. If $(X,d_F)$ constitutes a complete
and separable metric space such that $\mu$ is a Borel probability
measure on $(X,d_F)$, we call $\mathcal{D}$ a \emph{geometric dataset}
(GD). The aforementioned properties are satisfied when it holds that
$0<|X|,|F|<\infty$ and that $F$ can discriminate all data points,
i.e., $d_F(x,y) >0$ for all $x,y \in X$ with $x \neq y$. 

Two geometric datasets $\mathcal{D}_1\coloneqq (X_1,F_1,\mu_1),
\mathcal{D}_2\coloneqq (X_2,F_2,\mu_2)$ are isomorphic if there exists a
bijection\linebreak $\phi:X_1 \to X_2$ such that $\overline{F_2} \circ \phi =
\overline{F_1}$ and $\phi_*(\mu_1) = \mu_2$, where
$\phi_*(\mu_1)(B)\coloneqq\mu_1(\phi^{-1}(B))$ is the
\emph{push-forward measure} and the closures are taken with respect to
point-wise convergence. From this point on we identify a geometric
dataset with its isomorphism class. The triple
$(\{\emptyset\},\mathbb{R},\nu_{\{\emptyset\}})$ represents the
\emph{trivial geometric dataset}. 

\citet{pestov07} defines the curse of dimensionality as ``[...] a name
given to the situation where all or some of the important features of a dataset
sharply concentrate near their median (or mean) values and thus become
non-discriminating. In such cases, X is perceived as intrinsically
high-dimensional.'' Thus, the ID estimator aims to compute to which extent the
features allow to discriminate different data points. For a specific feature $f
\in F$, \citet{hanika22} therefore defines the \emph{partial diameter} of $f$ with regard to a
specific $\alpha \in (0,1)$ such that it displays to which extent $f$ can discriminate
subsets with minimal measure $1-\alpha$, i.e., via
\begin{equation*}
  \partdiam(f_*(\nu), 1{-}\alpha) = \inf
  \{\diam(B)~|~B \subseteq \mathbb{R} \text{ Borel, } \nu(f^{-1}(B)) \geq 1 -
    \alpha\},
\end{equation*} 
where $\diam(B) \coloneqq \sup_{a,b \in B}|a - b|$. The \emph{observable diameter} with
respect to $\alpha$ then defines to which extent $F$ can discriminate points
with minimum measure 1-$\alpha$ by being defined as the supremum of the partial
diameter of all $f \in F$, i.e, $\obsdiam(\mathcal{D}, - \alpha) \coloneqq \sup_{f \in F}\partdiam(f_*(\mu),1 -
\alpha)$. To observe the discriminability for different minimal
measures $\alpha$, the \emph{discriminability} $\Delta(\mathcal{D})$ of
$\mathcal{D}$ and the \emph{intrinsic dimension} $\partial(\mathcal{D})$ are
defined via
\begin{equation}
  \label{eq:delta_int}
\partial(\mathcal{D})\coloneqq \frac{1}{\Delta(\mathcal{D})^2},\quad
\text{where}\quad   \Delta(\mathcal{D})\coloneqq \int_0^1\obsdiam(\mathcal{D}, -\alpha)~d\alpha.
\end{equation}

In other words, lower values of intrinsic dimensionality correspond to
geometric datasets with points that can be better discriminated by the
given set of feature functions.  This intrinsic dimension function is,
in principle, applicable to a broad variety of geometric data, such as
metric data, graphs or images. This applicability arises from the
possibility to choose suitable feature functions which reflect the
underlying data structure. The appropriate choice of feature functions
is part of~\cref{sec:gnns}. Furthermore, the ID $\partial (\mathcal{D})
= \frac{1}{\Delta(\mathcal{D})^{2}}$ respects the formal axiomatization~\citep{hanika22} for
ID functions, informally:
\begin{itemize}
  \item \textbf{Axiom of concentration:} A sequence of geometric datasets converges against the
  constant dataset (meaning having no chance to separate data points!), if and
  only if their IDs diverge against infinity.
\item \textbf{Axiom of feature antitonicity:} If dataset $\mathcal{D}_1$ has
more feature functions then $\mathcal{D}_0$ (i.e. having potentially more
information to separate data points), it should have a lower intrinsic
dimension. 
\item \textbf{Axiom of continuity:} If a sequence of geometric datasets converge against
a specific geometric dataset, the sequence of the IDs should converge against
the ID of the limit geometric dataset.
\item \textbf{Axiom of geometric order of divergence:} If a sequence of geometric
datasets converges against the constant dataset, its IDs should diverge against
infinity with the same order as $\frac{1}{\Delta(\mathcal(D))^2}$ does.
\end{itemize}

\subsection{Intrinsic Dimension of Finite Data}
\label{sec:id_finite}

We want to apply~\cref{eq:delta_int} to real-world data. In the following, let
$\mathcal{D}=(X,F,\nu)$ such that $0<|X| <\infty$ and $0<|F| <\infty$
and let
$\nu$ be the normalized counting measure on $X$, i.e.,
$\nu(M)\coloneqq\frac{|M|}{|X|}$ for $M\subseteq X$. In this case, it
is possible to compute the partial diameter and~\cref{eq:delta_int},
as we show in the following.
Let $\alpha \in (0,1)$ and let $c_\alpha\coloneqq \lceil
|X|(1-\alpha) \rceil$. The following arguments were already hinted in
previous work~\citep{hanika22}, yet not formally discussed or
proven.
\begin{lemma}
 \label{lemma:partdiam}
  For $f \in F$ it holds that
  \begin{equation*}
    \partdiam(f_*(\nu), 1 - \alpha) = \min_{|M|=c_\alpha}\max_{x,y \in M}|f(x) -
    f(y)|.
  \end{equation*}
\end{lemma}

\begin{proof}
  It holds that
  \begin{align*}
  \partdiam(f_*(\nu), 1{-}\alpha) &= \inf
  \{\diam(B)~|~B \subseteq \mathbb{R} \text{ Borel, } \nu(f^{-1}(B)) \geq 1 -
    \alpha\}\\
    &=\inf
  \{\diam(B)~|~B \subseteq \mathbb{R} \text{ Borel,} |\{x \in X~|~f(x)
  \in B\}| {\geq} c_\alpha\}.  
  \end{align*}
  We have to show that
  
  \begin{align}
    \label{eq:a}
    \begin{split}
      \inf\{\diam(B)~|~B \subseteq \mathbb{R} \text{ Borel, } |\{x \in
      X~|~f(x)
      \in B\}| \geq c_\alpha\} = \\
      \min \{\max_{x,y \in M} |f(x) - f(y)|~|~M \subseteq X, |M| \geq
      c_\alpha\}.
    \end{split}
  \end{align}
  ``$\leq:$'' We show that $\{\diam(B)~|~B \subseteq \mathbb{R} \text{ Borel, } |\{x \in
    X~|~f(x)
    \in B\}| \geq c_\alpha\} \supseteq
    \{\max_{x,y \in M} |f(x) - f(y)|~|~M \subseteq X, |M| \geq
    c_\alpha\}$. Let $z\coloneqq \max_{x,y \in M}|f(x) - f(y)|$ such that $M \subseteq
    X$ with $|M| \geq c_\alpha$. Without loss of generality we assume
    that
    \begin{equation*}
      \forall x \in X: (\exists m_1, m_2 \in M: f(m_1)\leq f(x) \leq
      f(m_2) \implies x \in M).
    \end{equation*}
    Let $b\coloneqq \max_{x \in M} f(x)$, $a\coloneqq \min_{x \in M}
    f(x)$, then $M = \{x \in X~|~f(x) \in [a, b]\}$. Hence,
    $z=b-a\in \{\diam(B)|B \subseteq \mathbb{R} \text{ Borel,} |\{x \in
      X|f(x)
      \in B\}| \geq c_\alpha\}$.

 ``$\geq:$'' Let $B \subseteq \mathbb{R}$ be Borel with $|\{x \in
X~|~f(x)
\in B\}| \geq c_\alpha$. Furthermore, let $M
 \coloneqq \{x \in X~|~f(x) \in B\}$. It holds that $\diam(B)=\sup_{x,y
 \in B}|x - y| \geq \max_{x,y \in M}|f(x) - f(y)|$ because of the
choice of $M$. As $B$ was chosen arbitrarily, it follows ``$\geq$''.

Finally, we need that 
\begin{equation*}
  \min \{\max_{x,y \in M}|f(x) -
  f(y)|~|~|M| \geq c_\alpha\} = \min \{\max_{x,y \in M}|f(x) -
  f(y)|~|~|M| = c_\alpha\}.
\end{equation*}
``$\leq$'' follows directly from the fact that $\{\max_{x,y \in M}|f(x) -
f(y)|~|~|M| \geq c_\alpha\} \supseteq \{\max_{x,y \in M}|f(x) -
f(y)|~|~|M| = c_\alpha\}$. ``$\geq$'' follows from the fact that
for every $|M| \geq c_\alpha$ and for every $N \subseteq M$ with
$|N| = c_\alpha$ the following equation holds: $\sup_{x,y \in M} |f(x) - f(y)|
\geq \sup_{x,y \in N} |f(x) - f(y)|$.
\end{proof}

This lemma allows for a more tractable formula for the computation
of the partial diameter of a finite GD. That in turn enables the
following theorem.

\begin{theorem}
  It holds that
  \begin{equation}
    \label{eq:delta}
    \Delta(\mathcal{D}) = \frac{1}{|X|}\sum_{k=2}^{|X|}\max_{f \in F}\mathop{\min_{M
        \subseteq X}}_{|M|=k}\max_{x,y \in M}|f(x) - f(y)|.
  \end{equation}
\end{theorem}

\begin{proof}
  Let $g:(0,1) \to \mathbb{R}, \alpha \mapsto \max_{f \in F}
  \min_{M
      \subseteq X,|M|=c_\alpha}\max_{x,y \in M}|f(x) -
  f(y)|$. Because of \cref{lemma:partdiam} we know that $\Delta({\mathcal{D}})
  = \int_0^1 g(\alpha)~d\alpha$. The function $g$ is a step function
  which can be expressed for each $\alpha \in
  (0,1)$ via
  \begin{equation*}
    g(\alpha)=\sum_{k=1}^{|X|}\mathbbm{1}_{\left(\frac{|X|-k}{|X|},\frac{|X|+1-k}{|X|}\right)}(\alpha)\max_{f \in F}\mathop{\min_{M
        \subseteq X}}_{|M|=k}\max_{x,y \in M}|f(x) - f(y)|
  \end{equation*}
  almost everywhere.
  Hence, \cref{eq:delta} follows from the
  definition of the Lebesgue-Integral with the fact that $\min_{M
      \subseteq X,|M|=1}\max_{x,y \in M}|f(x) - f(y)|=0$.
\end{proof}

In general, the addition of features should lower the ID since we have
additional information that helps to discriminate the data. However, there are
certain features that are not helping to further discriminate data points. These
are for example:
\begin{enumerate}
\item Constant features. This is due to the fact that for a constant feature $f$ it always holds
for all $ M \subseteq X$ that $\max_{x,y \in M}|f(x) - f(y)|=0$.
\item Permutations of already existing features. Let $\tilde{f}: X \to
\mathbb{R}$ have the form $\tilde{f} = f \circ \pi$ with $\pi : X \to X$ being a
permutation on $X$. Then there exist for all $M \subseteq X$ with $|M|=k$ a set
$N \subseteq X$ with $|N|=k$ and $\max_{x,y \in M}|f(x) - f(y)| = \max_{x,y \in
N}|\tilde{f}(x) - \tilde{f}(y)|$ and vice versa.
\end{enumerate}
Thus, we have the following Lemma.

\begin{lemma}
Let $\mathcal{D}=(X, F, \mu)$ be a finite geometric dataset. Furthermore, let
$\hat{F}$ be a set of constant functions $X \to \mathbb{R}$ and let $\tilde{F}$
be a set of functions $X \to \mathbb{R}$ such that there exist for each
$\tilde{f} \in \tilde{F}$ a $f \in F$ and a permutation $\pi: X \to X$ with
$\tilde{f} = f \circ \pi$. Let $\mathcal{E}\coloneqq(X, F \cup \hat{F} \cup
\tilde{F}, \mu)$. Then it holds that $\partial(D) = \partial(E)$.
\end{lemma}

\subsection{Computing the Intrinsic Dimension of Finite Data}
In this section we will propose an algorithm for computing the ID
based on~\cref{eq:delta}. For this, given a finite geometric dataset
$\mathcal{D}$, we use the shortened notations $\phi_{k,f}(\mathcal{D})\coloneqq \min_{M \subseteq
  X,|M|=k}\max_{x,y \in M}|f(x) - f(y)|$ and 
$\phi_{k}(\mathcal{D})\coloneqq \max_{f \in
  F}\phi_{k,f}(\mathcal{D})$. Then, \cref{eq:delta} can be written as

\begin{equation}
  \label{eq:delta-short}
  \Delta(\mathcal{D}) = \frac{1}{|X|}\sum_{k=2}^{|X|}
  \phi_{k}(\mathcal{D})=\frac{1}{|X|}\sum_{k=2}^{|X|} \max_{f
    \in F}\phi_{k,f}(\mathcal{D}).
\end{equation}
The straightforward computation of \cref{eq:delta-short} is
hindered by the task to iterate through all subsets $M \subseteq X$ of
size $k$. This yields an exponential complexity with
respect to $|X|$ for computing $\Delta(\mathcal{D})$. We can overcome this
towards a quadratic computational complexity in $|X|$ using the
following concept.

\begin{definition}(Feature Sequence)
  For a feature $f \in F$ let $l_{f,\mathcal{D}} \in \mathbb{R}^{|X|}$ be the increasing
  sequence of all values $f(x)$ for $x \in X$. We call $l_{f,\mathcal{D}}=(l_1^{f,\mathcal{D}},\dots,l_{|X|}^{f,\mathcal{D}})$ the
  \emph{feature sequence} of $f$.
\end{definition}

Using these sequences, the following lemma allows us to efficiently
compute $\phi_{k,f}(\mathcal{D})$.

\begin{lemma}
  \label{lemma:k}
  For $k \in \{2,\dotsc, |X| \}, f \in F$ and $l_{f,\mathcal{D}}$, it holds that
  \begin{equation*}
    \phi_{k,f}(\mathcal{D}) = \min \{l_{k+j}^{f,\mathcal{D}}-l_{1+j}^{f,\mathcal{D}}~|~j \in \{0,\dots,|X|-k\}\}.
  \end{equation*}
\end{lemma}

\begin{proof}
For all $j \in \{0,\dots,|X|-k\}$ there exist $M \subseteq X$ with $|M|=k$ and
$l_{k+j}^{f,\mathcal{D}}-l_{1+j}^{f,\mathcal{D}}=\max_{x,y \in M}|f(x)-f(y)|$.
Thus, it is sufficient to show $\phi_{k,f}(\mathcal{D}) \in
\{l_{k+j}^{f,\mathcal{D}}-l_{1+j}^{f,\mathcal{D}}~|~j \in \{0,\dots,|X|-k\}\}$.
Choose $M \subseteq X$ with $|M|=k$ such that $\phi_{k,f}(\mathcal{D})=\max_{x,y
\in M}|f(x) - f(y)|$ holds. Furthermore, let $l^M\coloneqq(l_1^M,\dots,l_k^M)$
be the increasing sequence of values $f(m)$ for $m \in M$ and let $j \in
\{0,\dots,|X|-k\}$ such that $l_1^M=l_{1+j}^{f,\mathcal{D}}$. Since $l^M$ is an
ordered sequence of which each element is also an element of the ordered
sequence $l_{f,\mathcal{D}}$, it holds that $l_k^M \geq l_{k+j}^{f,\mathcal{D}}$
and thus $l_k^M - l_1^M\geq l^{f,\mathcal{D}}_{j+k}- l_{j+1}^{f,\mathcal{D}}$.
There is an $N \subseteq X$ with size $k$ such that $\max_{x,y \in
N}|f(x)-f(y)|=l_{k+j}^{f,\mathcal{D}}- l_{k+1}^{f,\mathcal{D}}$. Since $M
\subseteq X$ is % chosen a
  % the subset of $X$
  of size $k$ such that
  $\max_{x,y \in M}|f(x) - f(y)|$ is minimal, it follows
  $l_k^M - l_1^M = \max_{x,y \in M}|f(x) - f(y)| \leq \max_{x,y \in
    N}|f(x)-f(y)|=l_{k+j}^{f,\mathcal{D}}- l_{k+1}^{f,\mathcal{D}}$, hence
  $ \phi_{k,f}(\mathcal{D})=\max_{x,y \in M}|f(x) - f(y)|=l_k^M - l_1^M
  = l_{k+j}^{f,\mathcal{D}}- l_{k+1}^{f,\mathcal{D}}$.
\end{proof}
 \begin{algorithm}[t]
   \caption{The pseudocode to compute $\Delta(\mathcal{D})$ for a
     finite geometric dataset $\mathcal{D}=(X,\mu,F)$.}
    \label{alg:algo1}
  \DontPrintSemicolon \SetKwInOut{Input}{Input} \SetKwInOut{Output}{Output}
  \Input{Finite geometric dataset $\mathcal{D}=(X,\mu,F)$.}
  \Output{$\Delta(\mathcal{D})$}
  \BlankLine
  \ForAll{$f$ in $F$}
  {Compute feature sequence $l_{f,\mathcal{D}}$.}
  $\Delta(\mathcal{D}) = 0$\\
  \ForAll{$k$ in $\{2,\dots,|X|\}$}
  {\ForAll {$f$ in $F$}
  {$\phi_{k,f}(\mathcal{D}) = \min_{j \in
      \{0,\dots,|X|-k\}}l_{k+j}^{f,\mathcal{D}}-l_{1+j}^{f,\mathcal{D}}$.}
  $\Delta(\mathcal{D}) += \max_{f \in F}\phi_{k,f}(\mathcal{D})$}
$\Delta(\mathcal{D}) = \frac{1}{|X|}\Delta(\mathcal{D})$\\
\Return $\Delta(\mathcal{D})$
\end{algorithm}

To sum up, \cref{lemma:k} enables the efficient computation of
$\phi_{k,f}(\mathcal{D})$ 
via a sliding window, i.e., by using only pairs of elements
$(l_{1+j}^{f,\mathcal{D}},l_{k+j}^{f,\mathcal{D}})$. The algorithm
based on this is shown in~\cref{alg:algo1}. We want to provide a brief
description of the most relevant steps. In Line 4 we iterate through
the sizes of $X$ by setting $k\in\{2, \dots, |X|\}$ in order to
compute $\phi_k(\mathcal{D})$ in Lines 6 and 7. For this we
also need to iterate over all $f \in F$ (Line 5) to compute the
necessary values of $\phi_{k,f}(\mathcal{D})$ in Line 6. For a given
$f \in F, k \in \{1, \dots, |X|\}$, Line 6 consumes $|X|-k+1$
subtraction operations. Assuming that computing feature values can be
done in constant time, the runtime for computing $\Delta(\mathcal{D})$
from the feature sequences is
$\mathcal{O}(|F|\sum_{k=2}^{|X|}|X|-k+1)=\mathcal{O}(|F|\sum_{k=1}^{|X|-1}k)=\mathcal{O}(|F|\,|X|^2)$. The
creation of all feature sequences requires
$\mathcal{O}(|F||X|\log(|X|))$ computations , which is negligible
compared to the aforementioned complexity. Thus, \cref{alg:algo1} has
quadratic complexity with respect to $|X|$.  Therefore,
\cref{alg:algo1} is a straightforward and easy to implement solution
for the computation of the ID. However, its quadratic runtime is
obstructive for the application in large-scale data problems, which
raises the necessity for a modification. We will present such a
modification in the following
section.

\section{Intrinsic Dimension for Large-Scale Data}
\label{sec:approx}
In order to speed up the computation of the ID we
modify~\cref{alg:algo1} with regard to the accuracy of the
result. Hence, we settle for an efficiently computable approximation
of the ID. To give an overview over the necessary steps, we will
\begin{enumerate}
\item approximate the ID by replacing $\{2,\dots,|X|\}$ in Line 4
  of~\cref{alg:algo1} with a smaller subset $S\subseteq
  \{2,\dots,|X|\}$, which we represent by $S\coloneqq\{s_1,\dots,
  s_l\}$. For all $k \not\in S$, we will use
  $\{\phi_{s_1}(\mathcal{D}),\dots,\phi_{s_l}(\mathcal{D})\}$ to estimate $\phi_k(\mathcal{D})$. This will eventually lead to two approximations of the ID,
  an underestimation and an overestimation.
\item compare the upper and lower approximation to provide an error bound of these approximations with respect to
  the exact ID. This error bound can be computed without knowing the exact ID.
\item argue how, the computation of the exact ID can be sped up
  with the help of knowledge about $\phi_{s_i}(\mathcal{D})$ for all $s_i \in
  S$. For this, we will in particular show that we can replace for all $k \in
  \{2,\dots,|X|\} \setminus M$ the set $F$ with a subset
  $\hat{F}$, see Line 5-6 of~\cref{alg:algo1}.
\item derive a formula which estimates the amount of computation cost
  which is saved by using only subsets of $F$ for the computation of
  the ID. This information can be used to estimate and decide whether
  the exact computation of the ID is computational feasible for a
  specific dataset.
\end{enumerate}
The ensuing algorithm is shown in~\cref{alg:algo2}. The underlying
theory that justifies it is presented in the following. This theory will be
based on the monotonicity of $i \mapsto \phi_{i,f}(\mathcal{D})$.

\begin{theorem}
  \label{monoton}
  For $m > n \geq 2$ and $f \in F$ the following statements hold.
  \begin{enumerate}
    \item $\phi_{m,f}(\mathcal{D}) \geq
    \phi_{n,f}(\mathcal{D})$,
    \item $\phi_m(\mathcal{D}) \geq \phi_n(\mathcal{D})$.
  \end{enumerate}
\end{theorem}
\begin{proof}
The second inequality follows directly from the first one. Since per definition
$\phi_{m,f}(\mathcal{D})=\min_{M \subseteq X,|M|=m}\max_{x,y \in M}|f(x) -
f(y)|$ and also $ \phi_{n,f}(\mathcal{D})= \min_{N \subseteq X, |N|=n}\max_{x,y
\in N}|f(x) - f(y)|$, we need to show that for each $M \subseteq X$ with $|M|=m$
there exist $N \subseteq X$ with $|N|=n$ and $\max_{x,y \in M}|f(x) - f(y)| \geq
\max_{x,y \in N}|f(x) - f(y)|$. It is sufficient to show that for $n=m-1$.
Choose $x_1,x_2 \in M$ such that $\max_{x,y \in M}|f(x) - f(y)|=|f(x_1) -
f(x_2)|$. As $|M|>2$ we find $x_3 \in M \setminus \{x_1, x_2\}$. Let $N
\coloneqq M \setminus \{x_3\}$. It holds that $\max_{x,y \in M}|f(x) - f(y)| =
|f(x_1) - f(x_2)| = \max_{x,y \in N}|f(x) - f(y)|$.
\end{proof}

\begin{algorithm}[t]
  \caption{The pseudocode to compute
    $\Delta_{s,-}(\mathcal{D}),\Delta_{s,+}(\mathcal{D}),\Delta(\mathcal{D})$
    for a finite GD $\mathcal{D}=(X,\mu,F)$.}
  \label{alg:algo2}
  \DontPrintSemicolon \SetKwInOut{Input}{Input}
  \SetKwInOut{Output}{Output} \Input{Finite GD
    $\mathcal{D}=(X,\mu,F)$, support sequence
    $s=(2=s_1,\dots,s_l=|X|)$, \emph{exact} (Boolean)}
  \Output{$\Delta_{s,-}(\mathcal{D}),\Delta_{s,+}(\mathcal{D}),\Delta(\mathcal{D})$}
  \BlankLine \ForAll{$f$ in $F$} {Compute feature sequence $l_{f,\mathcal{D}}$.}
  $\Delta(\mathcal{D}) = 0$\\
  $s_0=1$ \\
  $\phi_{s_0}(\mathcal{D})=0$ \\
  \ForAll(\tcp*[f]{Iterate over support sequence indices}){$i$ in $\{1,\dots,l\}$}{\ForAll {$f$ in $F$}
    {$\phi_{s_i,f}(\mathcal{D}) = \min_{j \in
        \{0,\dots,|X|-s_i\}}l_{s_i+j}^{f,\mathcal{D}}-l_{1+j}^{f,\mathcal{D}}$
    \tcp*{Compute $\phi_{s_i,f}(\mathcal{D})$ with~\cref{lemma:k}}}
      $\phi_{s_i}(\mathcal{D}) = \max_{f \in
        F}\phi_{s_i,f}(\mathcal{D})$\\
      $F_{s_{i-1}} = \{f \in F~|~ \phi_{s_{i},f}(\mathcal{D}) >
      \phi_{s_{i-1}}(\mathcal{D})\} $\tcp*{Compute $F_{s_{i-1}}$
        for~\cref{lemma:order}}
    $\Delta(\mathcal{D}) += \phi_{s_i}(\mathcal{D})$}
  $\Delta_{s,-}(\mathcal{D})= \frac{1}{|X|}
  (\Delta(\mathcal{D}) +
  \sum_{i=1}^{l-1}\sum_{s_i<j<s_{i+1}}\phi_{s_i}(\mathcal{D}))$\tcp*{Compute
  lower ID from~\cref{def:support}}
  $\Delta_{s,+}(\mathcal{D})= \frac{1}{|X|}(
  \Delta(\mathcal{D}) +
  \sum_{i=1}^{l-1}\sum_{s_i<j<s_{i+1}}\phi_{s_{i+1}}(\mathcal{D}))$\tcp*{Compute
    upper ID from~\cref{def:support}}
  \vspace{.5cm}
  \tcp{Approximation finished. Continue with exact
    computation, if desired.}
  \If{exact}{
  \ForAll{$i$ in $\{1,\dots l\}$}
  {\ForAll(\tcp*[f]{Iterate through all indices between two support elements}){$s_i<j<s_{i+1}$}
    {\ForAll(\tcp*[f]{Only iterate through $F_{s_i}$ because of~\cref{lemma:order}}){$f$ in $F_{s_i}$}{$\phi_{s_i,f}(\mathcal{D}) = \min_{j \in
          \{0,\dots,|X|-s_i\}}l_{s_i+j}^{f,\mathcal{D}}-l_{1+j}^{f,\mathcal{D}}$.}
         $\phi_{j}(\mathcal{D}) = \max(\{\phi_{j,f}(\mathcal{D})~|~f \in F_i\} 
   \cup \{\phi_{s_i}\})$\tcp*{Use~\cref{lemma:order}}
   $\Delta(\mathcal{D}) += \phi_{j}(\mathcal{D})$\\
 }}
 $\Delta(\mathcal{D}) = \frac{1}{|X|}\Delta(\mathcal{D})$\\
  \Return
  $\Delta_{s,-}(\mathcal{D}),\Delta_{s,+}(\mathcal{D}),\Delta(\mathcal{D})$}
\Return $\Delta_{s,-}(\mathcal{D}),\Delta_{s,+}(\mathcal{D})$
\end{algorithm}

\subsection{Computing Intrinsic Dimension via Support Sequences}
Equipped with~\cref{monoton}, we can bound $\Delta(\mathcal{D})$ and thus the
intrinsic dimension through computing $\phi_{s_i}$ for a few
$2=s_1<s_2\dots<s_l=|X|$.

\begin{definition}(Support Sequences and Upper / Lower ID)\label{def:support}
  Let $s=(2=s_1, \dots,s_{l-1},s_l=|X|)$ be a strictly increasing and
  finite sequence of natural numbers. We call $s$ a \emph{support
    sequence} of $\mathcal{D}$. We additionally define
  \begin{align}
    \label{eq:u_p_delta}
    \begin{split}
      \Delta_{s,-}(\mathcal{D})&\coloneqq \frac{1}{|X|}
      \left(\sum_{i=1}^l\phi_{s_i}(\mathcal{D}) +
      \sum_{i=1}^{l-1}\sum_{s_i<j<s_{i+1}}\phi_{s_i}(\mathcal{D})\right),\\
      \Delta_{s,+}(\mathcal{D})&\coloneqq \frac{1}{|X|}
     \left( \sum_{i=1}^l\phi_{s_i}(\mathcal{D}) + \sum_{i=1}^{l-1}\sum_{s_i<j<s_{i+1}}\phi_{s_{i+1}}(\mathcal{D})\right)
    \end{split}
  \end{align}
  and call accordingly $\partial_{s,-}(\mathcal{D})\coloneqq
  \frac{1}{\Delta_{s,+}(\mathcal{D})^2}$ the \emph{lower intrinsic
    dimension} of $\mathcal{D}$ and $\partial_{s,+}(\mathcal{D})\coloneqq
  \frac{1}{\Delta_{s,-}(\mathcal{D})^2}$ the \emph{upper intrinsic
    dimension} of $D$.
\end{definition}

The governing idea is for $i \in \{1,\dots,l\}$ and $j$ with $s_i<j<s_{i+1}$ to
substitute $\phi_j(\mathcal{D})$ with
$\phi_{s_i}(\mathcal{D})$ or $\phi_{s_{i+1}}(\mathcal{D})$. With~\cref{monoton} this results in  lower and upper bounds for
$\Delta(\mathcal{D})$ and thus for the intrinsic ID. By comparing upper and lower bounds, we can approximate
the ID and estimate the approximation
error.

\begin{corollary}
  \label{bound}
  For support sequences $s$ holds
  $\Delta_{s,-}(\mathcal{D}) \leq \Delta(\mathcal{D}) \leq
  \Delta_{s,+}(\mathcal{D})$ and $\partial_{s,-}(\mathcal{D}) \leq \partial(\mathcal{D}) \leq
  \partial_{s,+}(\mathcal{D})$.
\end{corollary}

\begin{definition}[Approximation Error]\label{def:error}
  For a support sequence $s$ the \emph{(relative) approximation error}
  of $\partial(\mathcal{D})$ with respect to $s$ is given by
  \begin{equation*}
   \E(s,\mathcal{D}) \coloneqq \frac{\partial_{s,+}(\mathcal{D})-\partial_{s,-}(\mathcal{D})}{\partial_{s,-}(\mathcal{D})}.
  \end{equation*}
\end{definition}

With the computation of the upper and lower ID it is possible to
bound the error with respect to the ID $\partial(\mathcal{D})$. The following
corollary can be deduced from~\cref{bound} and~\cref{def:error}.

\begin{corollary}\label{cor:error}
  For a support sequence $s$ the following statements hold.
  \begin{enumerate}
  \item
    $\max\{\frac{\partial_{s,+}(\mathcal{D})-\partial(\mathcal{D})}{\partial(\mathcal{D})},
    \frac{\partial(\mathcal{D})-\partial_{s,-}(\mathcal{D})}{\partial_{s,-}(\mathcal{D})}\},
    \leq \E(s,\mathcal{D})$,
  \item $\max\{|\partial_{s,+}(\mathcal{D})-\partial(\mathcal{D})|,
    |\partial(\mathcal{D})-\partial_{s,-}(\mathcal{D})|\} \leq |\partial_{s,+}(\mathcal{D})-\partial_{s,-}(\mathcal{D})|$.
  \end{enumerate}
\end{corollary}

If the error of the approximation of a specific support sequence is not
sufficient, further elements can be added to the support sequence. Directly
from~\cref{eq:u_p_delta} follows the following corollary.

\begin{corollary}
Let $s=(2=s_1,\dots,s_l=|X|)$ be a support sequence and let
$\hat{s}=(s_1,\dots,s_{i},p,s_{i+1},\dots,s_l)$ with a support
sequence with an additional element $p$. Then it holds that
\begin{align*}
  \Delta(\mathcal{D})_{\hat{s},-} &=  \Delta(\mathcal{D})_{s,-} + \frac{\sum_{p\leq j<s_{i+1}}\left((\phi_p(\mathcal{D})) - \phi_{s_i}(\mathcal{D})\right)}{|X|},\\
  \Delta(\mathcal{D})_{\hat{s},+} &=  \Delta(\mathcal{D})_{s,+} + \frac{\sum_{s_i < j \leq p}\left((\phi_p(\mathcal{D})) - \phi_{s_{i+1}}(\mathcal{D})\right)}{|X|}.
\end{align*}
\end{corollary}

For a given support sequence $s$, \cref{cor:error} gives us an upper bound for
the error when $\partial_{s,+}(\mathcal{D})$ or $\partial_{s,-}(\mathcal{D})$
are used to approximate $\partial(\mathcal{D})$ without knowing
$\partial(\mathcal{D})$. Hence, we can compute (a lower bound) for the accuracy
when approximating the ID with~\cref{def:error}. As we can see
in~\cref{sec:sups}, \cref{sec:large} and~\cref{sec:random}, comparable small
support sequences lead to sufficient approximations.  Support sequences can also
be used to shorten the computation of the exact intrinsic dimension as the
following lemma shows.

\begin{lemma}\label{lemma:order}
 Let $s=(2=s_1,\dots,s_l=|X|)$ be a support sequence. Furthermore, let
 $i \in
 \{1,\dots,l-1\}$ and let $j \in\mathbb{N}$ with $s_i<j<s_{i+1}$. Let $F_{s_i}\coloneqq \{f \in F~|~
 \phi_{s_{i+1}, f}(\mathcal{D}) >
 \phi_{s_i}(\mathcal{D})\}$.
 Then it holds that
 \begin{equation*}
   \phi_{j}(\mathcal{D}) = \max(\{\phi_{j,f}(\mathcal{D})~|~f \in F_{s_i}\} 
   \cup \{\phi_{s_i}(\mathcal{D})\}).
 \end{equation*}
\end{lemma}

\begin{proof}
  ``$\geq$'' follows from~\cref{monoton} and the definition of
  $\phi_{j}(\mathcal{D})$. ``$\leq$'' holds because for $f \in
  F\setminus F_{s_i}$ it holds that $\phi_{j,f}(\mathcal{D}) \leq
  \phi_{s_{i+1},f}(\mathcal{D})$, due to~\cref{monoton}, and
  $\phi_{s_{i+1}, f}(\mathcal{D}) \leq \phi_{s_i}(\mathcal{D})$, due to
  the construction of $F_{s_i}$.
\end{proof}

Hence, given a specific $j$, it is possible to compute
$\phi_j(\mathcal{D})$ using a subset of $F$. Based on the particular
GD $\mathcal{D}$, this fact can considerably speed up the computation
of the ID of $\mathcal{D}$, as we will see in~\cref{sec:gnns}.

An algorithm to approximate and compute the ID through support sequences is
depicted in~\cref{alg:algo2}. This algorithm takes as input a GD $\mathcal{D}$
and a chosen support sequence $s$. A reasonable choice for support sequences is
discussed in~\cref{sec:sups}. The output is
$\Delta_{s,-}(\mathcal{D}),\Delta_{s,+}(\mathcal{D})$, and
$\Delta(\mathcal{D})$, if desired (Line 14). In Line 2, all feature sequences
are computed. In Line 6 to Line 11, $\phi_{s_i}(\mathcal{D})$ and $F_{s_{i-1}}$, as defined
in~\cref{lemma:order}, are computed. From Line 1 to Line 13, the feature
sequences and the lower and upper ID are computed. If desired, the exact
computation is done in Line 15 to Line 21. Here, we iterate for all support
elements (Line 15) through all ``gaps'' between them (Line 16) and compute
$\phi_{j}(\mathcal{D})$ using~\cref{lemma:order}~(Line 17 to Line 19).

\subsection{Estimating Computational Costs}
\label{sec:costs}

Let $s=(s_1,\dots,s_l)$ be a support sequence. After the computation of
$F_{s_1},\dots,F_{s_{l-1}},$ we can estimate how much computation steps we can avoid
in order to compute $\partial (\mathcal{D})$ with~\cref{alg:algo2} compared
to~\cref{alg:algo1}. Together with the error function $E(s)$, this estimation
can help us to decide if it is desirable to compute the exact value $\partial
(\mathcal{D})$ or leave it at $\partial_{s,-}(\mathcal{D})$ and
$\partial_{s,+}(\mathcal{D})$. This is done in the following manner. For a
specific $f \in F$, \cref{lemma:k} shows that the computation of
$\phi_{k,f}(\mathcal{D})$ requires $|X|-k+1$ different subtractions and to keep
the minimum value. Hence, the cost for computing $\Delta(\mathcal{D})$ and
therefore $\partial (\mathcal{D})$ via~\cref{alg:algo1} can be estimated via
$\mathcal{O} (|F| \sum_{k=2}^{|X|}(|X|-k+1)) = \mathcal{O}(|F|
\sum_{k=1}^{|X|-1} k )= \mathcal{O}(|F|(\frac{|X|^2 - |X|}{2}))$. However, if we
use~\cref{alg:algo2}, we solely have the cost to compute $\phi_{s_i}$. For all
values $j$ with $s_i<j<s_{i+1}$, our cost estimation is  $|F_{s_i}|( |X|-j+1)$.
Hence, for a given support sequence $s=(s_1,\dots,s_l)$, we can estimate how
many computations are saved using the following notions.

We address the \emph{naive computation costs} for computing the ID of
a GD with
  \begin{equation*}
    \costs (\mathcal{D}) \coloneqq |F|(\frac{|X|^2 - |X|}{2}).
  \end{equation*}
  In contrast, for a support sequence $s=(s_1,\dots,s_l)$ of $\mathcal{D}$, the \emph{computation
    costs} are
  \begin{equation}\label{eq:costs}
    \costs_s(\mathcal{D})\coloneqq (|F| \sum_{k=1}^l |X|-s_k+1) +
    \sum_{k=1}^{l-1}|F_{s_k}|\sum_{s_k <j< s_{k+1}}|X|-j+1.
  \end{equation}
  Hence, the \emph{saved costs} of $s$ are
  \begin{equation*}
    \scosts_s(\mathcal{D}) \coloneqq 1 - \frac{\costs_s(\mathcal{D})}{\costs (\mathcal{D})}.  
  \end{equation*}
Once we have computed $\phi_{s_i}(\mathcal{D})$ and $F_i$, depending
on the
saved costs, we can decide to discard the support sequence or to
continue further computations with it. Furthermore, using the error
estimation, we can decide to compute the exact ID or to settle with
the approximation.

\section{Intrinsic Dimension of Graph Data}
\label{sec:gnns}

\begin{table}[t]
  \centering
    \caption{Statistics of all datasets used in this work.}
  \label{tab:data}

  \begin{tabular}[h]{l|lll}
    &Nodes &Edges& Attributes \\
    \midrule
    PubMed &$ 19,717 $&$ 88,648$&$ 500$\\
    Cora &$ 2,708 $&$ 10,556 $&$ 1,433$ \\
    CiteSeer &$ 3,327 $&$ 9,104 $&$ 3,703$ \\
    ogbn-arxiv &$ 169,343 $&$ 1,166,243 $&$ 128$ \\
    ogbn-products &$ 2,449,029 $&$ 61,859,140 $&$ 100$ \\
    ogbn-mag &$ 1,939,743 $&$ 21,111,007 $&$ 128$ \\
    ogbn-papers100M & $111,059,956$ & $1,615,685,872$ & $128$
  \end{tabular}
\end{table}

\begin{figure}[t]
  \centering
  \includegraphics[width=\linewidth]{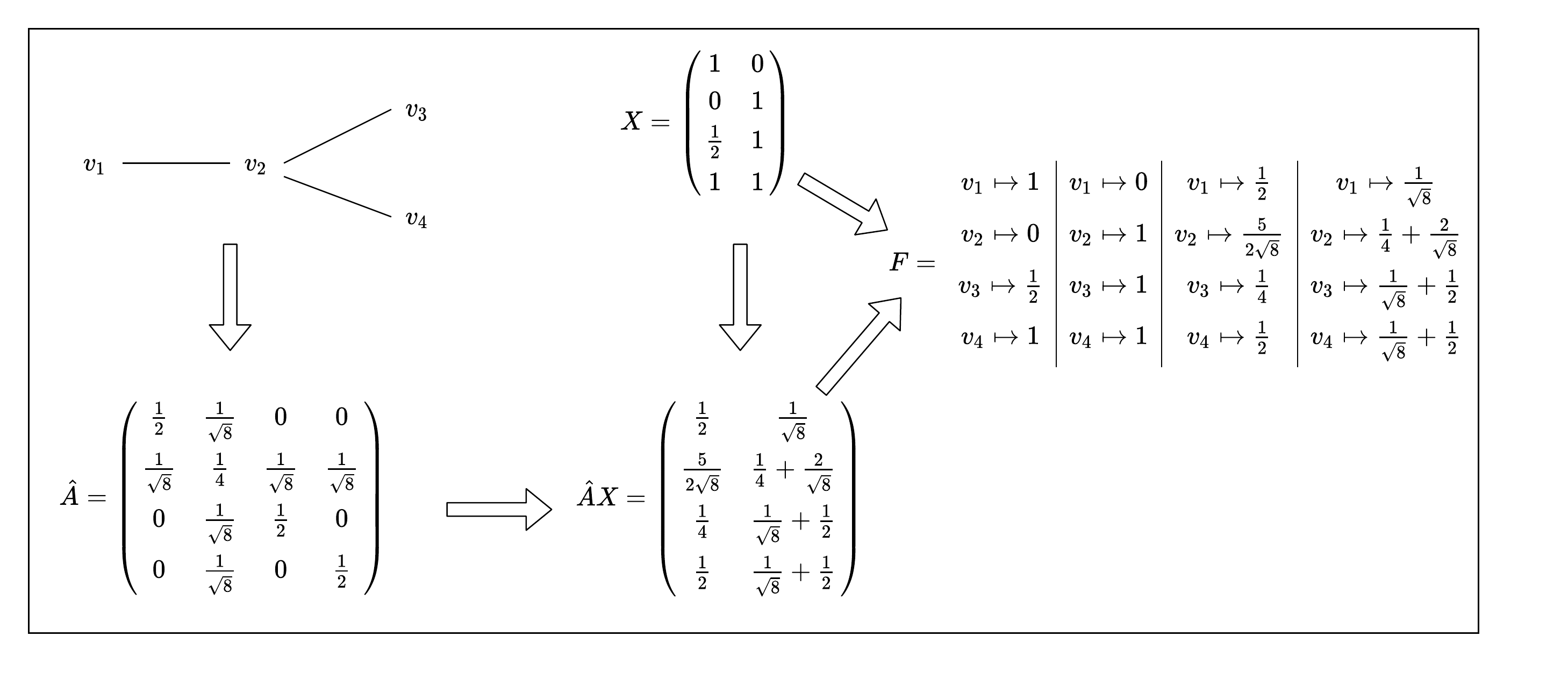}
  \caption{Example of an $k$-hop geometric dataset with $k=1$. Given
    are a graph and an attribute matrix $X$. Then, the normalized
    adjacency matrix $\hat{A}$ and then $\hat{A}X$ are computed. The
    feature set $F$ consists of
  the coordinate projections of $X$ and $\hat{A}X$. In the figure, every column
  after ``$F=$'' represents one $f \in F$. Note, that in this
example the normalization factor $\frac{1}{d_{\max}}$ is $1$.}
  \label{fig:graph}
\end{figure}

Graph data is of major interest in the realm of geometric learning and
beyond. In the following, a \emph{graph dataset} $D=(\mathcal{X},G)$ consists of
an undirected, unweighted graph $G=(V,E)$, where $V=\{v_1,\dots,v_n\}$
is a finite set of vertices, $E \subseteq \binom{V}{2}$ and $\mathcal{X} \in
\mathbb{R}^{n \times d}$ is a \emph{$d$-dimensional attribute matrix}. The row-vector
$\mathcal{X}_i=(\mathcal{X}_{i,1},\dots,\mathcal{X}_{i,d})$ is called the \emph{attribute vector} of
$v_i$.

Learning from such data is often done via \emph{graph neural networks}. The idea
is to extend common multi-layer perceptrons by a so called \emph{neighborhood
aggregation}, where internal representations of graph neighbors are combined at
specific layers. In earlier works, neighborhood aggregation is done at multiple
layers~\citep{kipf17, hamilton17, velickovic18}. Due to scalability, recent
approaches perform multiple iterations of neighborhood aggregation as a
preprocessing step and then use the aggregated features as combined
input~\citep{rossi20, sun21, zhang21}. To be more specific, the employed networks use inputs of the
form 
\begin{equation}
  \label{eq:input}
  (\mathcal{X}, \hat{A}\mathcal{X},\hat{A}^{2}\mathcal{X} \dots,
  \hat{A}^kX),  
\end{equation}
 where $\hat{A}$ is a \emph{transition matrix} that is
derived from the graph structure. The most common choice for such a
matrix is the normalized adjacency matrix, i.e., $\hat{A}_{i,j}
=(\sqrt{\deg(v_j)\deg(v_i)})^{-1}$ if $v_j \in N(v_i)$ and
$\hat{A}_{i,j}=0$ else.
Here, $N(v_i) \coloneqq \{v_j \in V~|~\{v_i,v_j\} \in E\} \cup \{v_i\}$ is
the set of neighbors of $v_i$ and $\deg(v_i) \coloneqq |N(v_i)|$ is
the node degree of $v_i$. The feature set of the following geometric dataset
corresponds to the input in~\cref{eq:input}.

\begin{definition}\label{def:k-hop}
  Let $k \in \mathbb{N}$ and $\hat{A}$ be the normalized adjacency
  matrix of a graph dataset $D$. Furthermore, let $d_{\max} \coloneqq \max_{j \in \{1,\dots,d\}}\max_{i,k \in
    \{1,\dots,n\}}|\mathcal{X}_{i,j} - \mathcal{X}_{k,j}|$. We call the set
  \[F_{D,k} \coloneqq\{v_i \mapsto \frac{1}{d_{\max}} (\hat{A}^m \mathcal{X})_{i,j}~|~m \in \{0,\dots,k\}, j \in
  \{1, \dots, d\}\}\] the \emph{$k$-hop feature functions} of $D$. Let $\nu$ be the normalized
  counting measure on $V$.
If
there exist for each $v_i,v_k \in V$ with $v_i \neq v_j$ elements $m \in \{0,\dots,k\}, j \in
  \{1, \dots, d\}$ such that $\frac{1}{d_{\max}} (\hat{A}^m \mathcal{X})_{i,j}
  \neq \frac{1}{d_{\max}} (\hat{A}^m \mathcal{X})_{k,j}$, then
  $\mathcal{D}_k=(V,F_{D,k},\nu)$ is a GD. We call it the \emph{$k$-hop geometric
    dataset} of $D$.
\end{definition}

Basic statistics of all seven graph datasets considered in the following sections
are depicted in~\cref{tab:data}. The statistics for \textbf{Cora},
      \textbf{PubMed} and \textbf{CiteSeer} were taken
      from PyTorch Geometric~\footnote{\url{https://pytorch-geometric.readthedocs.io/en/latest/modules/datasets.html\#torch\_geometric.datasets.Planetoid}}. The
    statistics of the OGB datasets were taken from the Open Graph
    Benchmark.~\footnote{\url{https://ogb.stanford.edu/docs/nodeprop}}
An example of a $k$-hop geometric dataset is depicted in~\cref{fig:graph}.
It is well-known that the normalized adjacency matrix $\hat{A} \in
\mathbb{R}^{n \times n}$ of a graph has a
spectral radius of 1. As $\hat{A}$
is symmetric, this yields $\|\hat{A}x - \hat{A}y\| \leq \|x -y\|$ for $x,y \in
\mathbb{R}^n$.
The significance of this property is for the respective computations,
however, limited, since it primarily leads to insights of the behavior of
the columns of $\mathcal{X}$ under multiplication with powers of $\hat{A}$. In
contrast, the attribute vectors of the vertices are represented via
the rows.
Moreover, we may point out that we are not considering the Euclidean
distances between attribute vectors, but differences between
coordinate values. Thus, the spectral radius of $\hat{A}$ does not
provide direct insights into $F_{D,k}$.

\subsection{Choosing Support Sequences}
\label{sec:sups}
\begin{figure}[t]
  \centering
  \includegraphics[width=\linewidth]{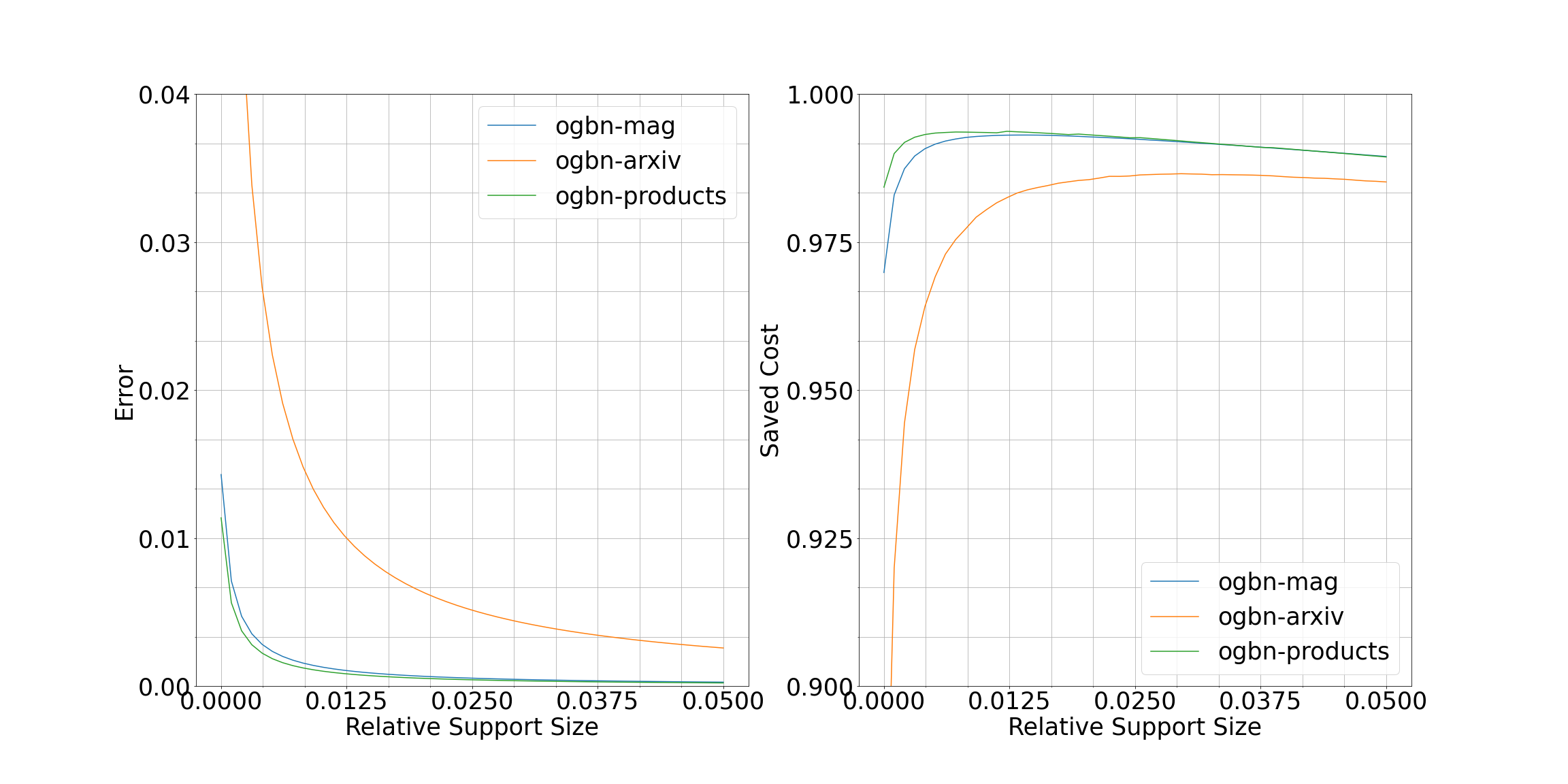}
\caption{Errors and saved costs for approximating and computing the intrinsic dimensionality for $2-hop$ geometric
datasets with different lengths of the support sequence.}
  \label{fig:cs}
\end{figure}

\cref{alg:algo2} relies on a proper choice for a support sequence $s$. To choose
$s$, two properties have to be considered. Namely, the length of
the support sequence and the spacing of the elements. Regarding
the second point, we decided to use log-scale spacing. To get such a support
sequence for a geometric dataset $\mathcal{D}=(X,F,\mu)$, we first choose a
geometric sequence $\hat{s}=(s_1,\dots s_l)$ of length $l$ from $|X|$ to $2$. We
derive the final support sequence $s$ from $s'=(\lfloor |X|+2 - s_1 \rfloor,
\dots, \lfloor |X| + 2 - s_l \rfloor)$ by removing duplicated elements.

In the following, we study the error and the saved costs for
different lengths $l$ of the support sequence. Here, for a geometric dataset, we investigate how
$\E(s, \mathcal{D})$ and $\scosts(s)$ vary for $s$ chosen  with
$l \in \{\lfloor0.001*|X|\rfloor, \lfloor0.002*|X|\rfloor, \dots,
\lfloor0.05*|X|\rfloor\}$. Here, if $l = \lfloor  r * |X|\rfloor $, we
call $r \in \mathbb{R}$ the \emph{relative support size} of the
resulting support sequence $s$. We experiment with
common benchmark datasets, namely
\textbf{ogbn-arxiv}, \textbf{ogbn-mag} and \textbf{ogbn-products} from
the Open Graph Benchmark~\citep{hu20, hu21}.  Since for
\textbf{ogbn-mag} only a subset of vertices is equipped with attribute
vectors, we generate the missing vectors via
metapath2vec~\citep{dong17}. For all datasets, we consider the $2$-hop geometric dataset.
The results are depicted in~\cref{fig:cs}.

\subsubsection{Results}
For all datasets, low errors and high saved costs
can be reached with a remarkably short support sequence. With relative
support sizes of under $0.015$ all datasets are approximated
with an accuracy of over $99 \%$. Furthermore, the saved costs for
sequences with comparable relative support sizes is over $0.98$. It
stands out, that for the larger datasets \textbf{ogbn-mag} and
\textbf{ogbn-products}, shorter sequences (relative to the size of the
dataset) lead to lower errors and higher saved costs then for
\textbf{ogbn-arxiv}. Our results further indicate, that a relative
support size between $0.01$ and $0.02$ is a reasonable range for
maximizing the saved costs. For longer support sequences, the saved
cost decrease while the error does not change dramatically, at least
for the $2$-hop geometric datasets of \textbf{ogbn-mag} and
\textbf{ogbn-products}. Note, that longer support sequence do not always lead to
a higher amount of saved costs. For longer support sequences $s$ the costs of
computing $\phi_k(\mathcal{D})$ for $k \not \in s$ decreases. However, the costs of computing
$\phi_{s_i}(\mathcal{D})$ for all elements $s_i \in s$ increase.

\subsection{Neighborhood Aggregation and Intrinsic
  Dimension}
\label{sec:ids}

\begin{table}[t]
  \caption{Intrinsic dimension and performances on classification
  tasks. In the upper table, we display IDs for all $k$-hop
  geometric datasets for $k \in \{0,\dots,5\}$. In the middle table,
  we display the ID estimated by the MLE baseline. In the lower table
  we display mean and standard derivations for test accuracy of a
  standard SIGN model on the classification tasks which belongs to
  the dataset.}
  \label{tab:k_hop_ids}
\centering\scalebox{0.8}{
  \begin{tabular}{l|p{2.5cm}p{2.5cm}p{2.5cm}p{2.5cm}p{2.5cm}p{2.5cm}}
    \backslashbox{Dataset}{$k$-hop}&$0$&$1$&$2$&$3$&$4$&$5$\\
    \midrule
      PubMed&$2542.3425$&$2336.6611$&$2077.5821$&$2077.0953$&$2077.0886$&$2077.0848$\\
      Cora&$6.2523$&$3.8324$&$3.6689$&$3.6627$&$3.6624$&$3.6623$\\
      CiteSeer&$22.3337$&$11.3166$&$10.2347$&$9.8134$&$9.5491$&$9.3795$\\
      ogbn-arxiv&$83.9160$&$31.4731$&$31.4731$&$31.4730$&$30.7370$&$30.3767$\\
    ogbn-products&$1,169,323.2496$&$1,169,044.4736$&$1,169,044.2216$&$1,169,044.2216$&$1,169,044.2216$&$1,169,044.2216$\\
    ogbn-mag&$2,311.3509$&$2,284.0290$&$2,284.0290$&$2,284.0290$&$2,284.0290$&$2,284.0290$
  \end{tabular}}
\\
\vspace{.3cm}
\centering\scalebox{0.8}{
  \begin{tabular}{l|p{2.5cm}p{2.5cm}p{2.5cm}p{2.5cm}p{2.5cm}p{2.5cm}}
    \backslashbox{Dataset}{$k$-hop}&$0$&$1$&$2$&$3$&$4$&$5$\\
    \midrule
      PubMed&$24.4623$&$24.7303$&$23.3924$&$22.2779$&$21.3495$&$20.5642$\\
      Cora&$30.6049$&$28.1785$&$19.9316$&$10.8186$&$9.2970$&$8.6155$\\
      CiteSeer&$58.9593$&$26.5031$&$16.5556$&$12.0495$&$9.3171$&$7.9572$\\
      ogbn-arxiv&$16.2948$&$19.8571$&$18.9068$&$18.2265$&$17.4905$&$16.9325$\\
    ogbn-products&$2.8694$&$4.7542$&$4.7950$&$4.7659$&$4.6943$&$4.6687$\\
    ogbn-mag&$30.7024$&$33.2848$&$31.5140$&$30.4844$&$29.9080$&$29.5956$
  \end{tabular}}
\\
\vspace{.3cm}
\centering\scalebox{0.8}{
\begin{tabular}{l|p{2.5cm}p{2.5cm}p{2.5cm}p{2.5cm}p{2.5cm}p{2.5cm}}
       \backslashbox{Dataset}{$k$-hop}&$0$&$1$&$2$&$3$&$4$&$5$\\
  \midrule
  PubMed&$.6850 \pm .0145$&.$7191 \pm .0123$ & $.7378 \pm .0362$
                                                  &.$7565 \pm .0165$
                                                      & $.7615 \pm
                                                        .0160$&
                                                                $.7571
                                                                \pm.0234$
  \\
  Cora & $.5329 \pm .0120 $ & $.7223 \pm .0117$ & $.7766 \pm
                                                  .0045$&$.7870 \pm
                                                         .0076$ &
                                                                  $.7917
                                                                  \pm
                                                                  .0084$
                                                          & $.7951
                                                            \pm
                                                            .0047$
  \\
  CiteSeer & $.4975 \pm .0075$& $.6165 \pm .0160$ & $.6530 \pm
                                                     .0101$ & $.6677
                                                              \pm
                                                              .0074$
                                                      & $.6695 \pm
                                                        .0085$ &
                                                                 $.6734
                                                                 \pm
                                                                 .0080$\\
  ogbn-arxiv & $.5341 \pm .0090$ & $.6572 \pm .0052$ & $.6903 \pm
                                                       .0056$ &
                                                               $.6917
                                                               \pm
                                                                .0074$
                                                      & $.6901 \pm
                                                        .0083$ &
                                                                 $.6890
                                                                 \pm
                                                                 .0051$
                                                                 \\
  ogbn-products & $.5969 \pm .0016$ & $.7204 \pm .0017$ & $.7590 \pm
                                                          .0017$ &
                                                                   $.7660
                                                                   \pm
                                                                   .0014$
                                                      & $.7678 \pm
                                                        .0022$ &
                                                                $.7687
                                                                \pm .0019$
  \\
  ogbn-mag & $.2712 \pm .0020$ & $.3635 \pm .0029$ & $.3879 \pm
                                                     .0030$ & $.3959
                                                              \pm
                                                              .0029$
                                                      & $.3983 \pm
                                                        .0050$ &
                                                                 $.4012
                                                                 \pm
                                                                 .0040$
  
\end{tabular}}
\end{table}

We study how the choice of $k$ affects the intrinsic
dimension value of the $k$-hop geometric dataset. For this, we compute the
intrinsic dimension for $k \in \{0,1,\dots,5\}$ for six datasets: the
three datasets mentioned above and \textbf{PubMed}, \textbf{Cora} and
\textbf{CiteSeer}~\citep{yang16}, which we retrieved from PyTorch
Geometric~\citep{fey19}.
Furthermore, we train GNNs which use the feature functions of
$k$-hop geometric datasets as information for training and
inference. This allows us to discover connections between the ID of
specific datasets with respect to the considered feature functions and
the performance of classifiers, which rely on these feature
functions. For this, we train SIGN models~\citep{rossi20} for $k \in
\{0,\dots,5\}$. Implementation details and parameter choices can be
found in~\cref{sec:setup}.

\subsubsection{Baseline Estimator}
\label{sec:baseline}

To investigate to which extent our ID function surpasses established ID
estimators with respect to estimating the discriminability of a dataset, we also
compute all ID values with the Maximum Likelihood Estimator
(MLE)~\citep{levina04}. This estimator is commonly used in the realm of deep
learning~\citep{pope20, ma18b, ma18}. For our experiments, we use the corrected
version proposed by~\citet{mackay05}. Note, that the MLE  is only applicable to
datasets $\mathcal{X} \in \mathbb{R}^{n \times d}$ and is thus not able to respect the neighborhood
aggregated feature functions. Hence, we incorporate the neighborhood information
of a $k$-hop dataset by concatenating feature vectors with the neighborhood
aggregated feature vectors. Due to performance reasons,  only subsets of the
data points are considered for \textbf{ogbn-mag} and \textbf{ogbn-products}.
More details to our usage of the MLE are discussed in~\cref{sec:impl_baseline}.

\subsubsection{Results}
\label{sec:agg_res}

We find that one iteration of neighborhood aggregation always leads to a huge drop
of the ID when using our ID function. However, consecutive iterations only lead to  a small decrease. For the datasets from OGB, some iterations lead to no drop
of the ID dimension at all. For \textbf{ogbn-mag}, only the
first iteration significantly decreases the ID, for \textbf{ogbn-products}, only the
first two iterations are relevant for decreasing the ID. It stands
out, that for \textbf{ogbn-arxiv}, the second and third iteration lead
to no significant decrease, but the fourth and fifth do.
The results for \textbf{PubMed} stand out. Here, the second iteration
of neighborhood aggregation leads to a comparable decrease as the
first one.

Considering the classification performances, the first iteration is
again the key factor, leading to a significant increase in
accuracy. As for the ID, the \textbf{PubMed} dataset behaves differently than the
other datasets: the second iteration of neighborhood aggregation
leads to a comparable increase in accuracy as the first.

The MLE ID behaves different. Here, no pattern of the first
iteration of aggregation being the key for decreasing the data
complexity is observed. For some datasets, the first rounds of feature
aggregation may even increase the intrinsic dimension. To sum up, our results indicate
that our ID is a better indicator for classification performance then
the MLE ID.

\subsection{Synthetic Data}
\begin{figure}[t]
  \centering
  \includegraphics[width=\linewidth]{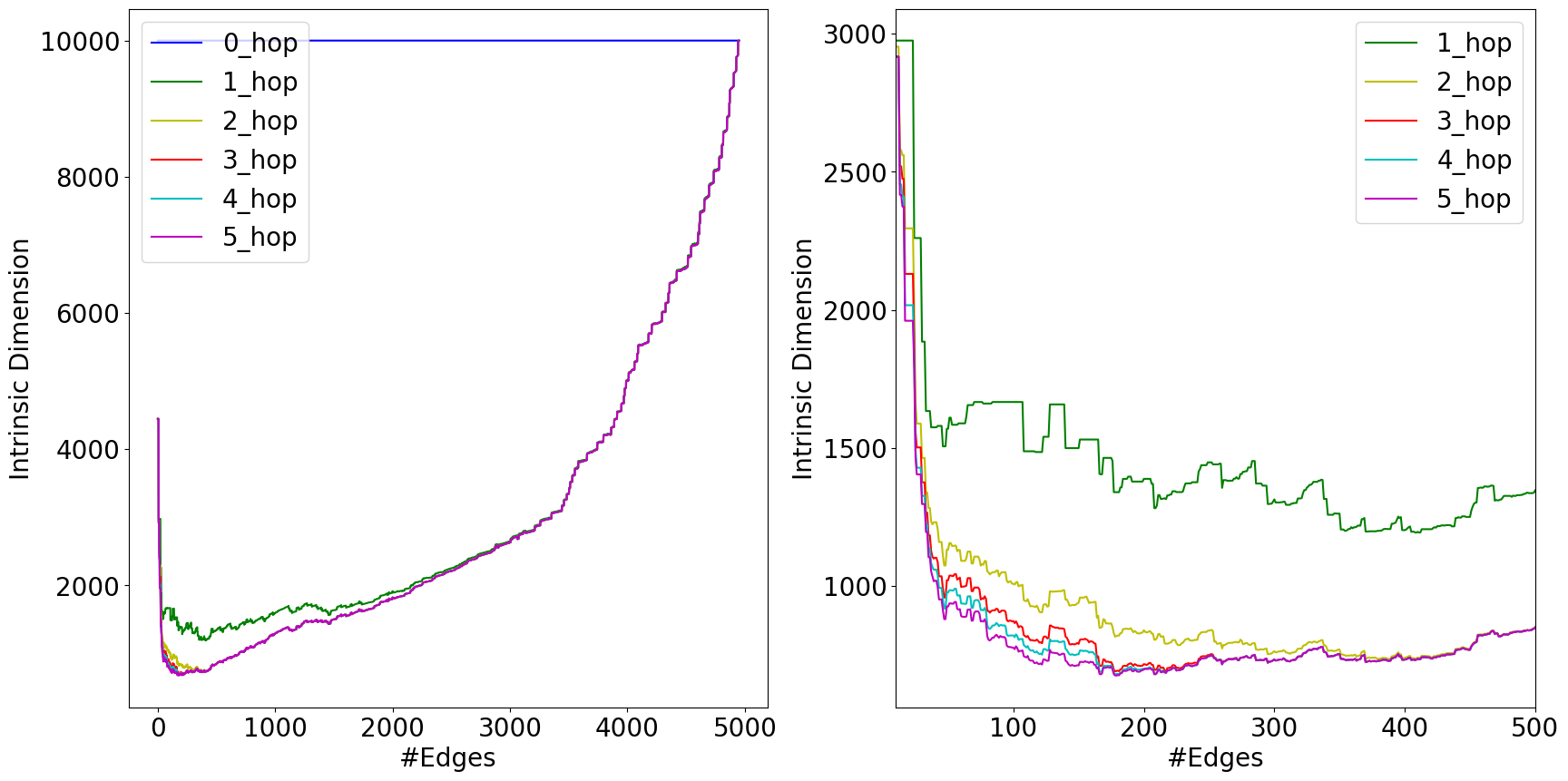}
  \caption{Intrinsic dimension of one-hot encoded features.
  The right plot is a zoom of the left one which hides the $0-$hop geometric dataset.}
  \label{fig:identity}
\end{figure}

To get further insights into the behavior of our ID notion, we now consider
$k-$hop geometric dataset for one-hot encoded graph data, i.e., $\mathcal{D}=(\mathcal{X},(V,E))$,
with $\mathcal{X}=\mathbb{I}_{|V|}$ where $\mathbb{I}_{|V|}$ is the $|V|$-dimensional
identity matrix. We consider the case of $V = \{1, \dots, 100\}$ and determine
the ID for $k \in \{0,\dots,5 \}$ for increasing edge sizes. To do so, we place
the $4950$ possible edges in a random order and add them step by step and
compute the ID notions for the $k-$hop geometric dataset. The results can be
found in~\cref{fig:identity}.

\subsubsection{Results}
The $0-$hop geometric dataset has an ID which does not depend on the amount of
edges. This is not surprising since it does not incorporate any graph
information. For all other $k$ values, the ID first sharply decreases and then
increases. This indicates that the addition of neighborhood aggregation is
particularly useful for graphs of moderate density. Here, the addition of
additional rounds of aggregation beyond the first one can further lower the ID.
For higher edge sizes, the ID difference between different $k$
values vanishes.  

\subsection{Approximation of Intrinsic Dimension on Large-Scale Data}
\label{sec:large}

To demonstrate the feasibility of our approach, we use it to
approximate the ID of the well known, large-scale
\textbf{ogbn-mag-papers100M} dataset. For this, we
construct the support sequence as in~\cref{sec:sups} with
$l=100.000$. The results are depicted in~\cref{tab:big}. On our \emph{Xeon
Gold System} with $16$ cores, approximating the ID of a $k$-hop
geometric dataset build from~\textbf{ogbn-mag-papers100M} is possible
within a few hours.
While the ID drops for every iteration of neighborhood aggregation,
the decrease becomes smaller. The ID of the
different $k$-hops can be differentiated by the approximation, i.e.,
$\partial_{s,-}(\mathcal{D}_i) > \partial_{s,+}(\mathcal{D}_{i+1})$
for $i \in \{0,\dots,4\}$. It stands out, that even for such a short
support sequence (compared to the size of the dataset), the observed
error is remarkably low. In detail, we can approximate the ID with an accuracy of
over $99.95 \%$. It is further remarkable, that the error does not change
significantly for different $k$. We observed this effect also for the
other datasets. Our results on \textbf{ogbn-papers100M} indicate,
that with short support sequences, we can sufficiently
approximate the ID of large-scale graph data.

\section{Errors of Random Data}
\label{sec:random}

\begin{table}[t]
  \caption{Approximation of intrinsic dimension for ogbn-papers100M.}
  \label{tab:big}
  \centering
  \begin{tabular}{l|llllll}
    \backslashbox{}{$k$}&0&1&2&3&4&5\\
    \midrule
    $\partial_{s,-}(\mathcal{D})$&$282.2380$& $171.7385$& $148.3323$ &
                                                                       $137.7662$ & $128.2751$ & $125.3418$\\
     $\partial_{s,+}(\mathcal{D})$&$282.3387$& $171.7997$& $148.3852$ &
                                                                       $137.8153$
                                & $128.3208$ & $125.3864$\\
    $\E(s,\mathcal{D})$&$0.0004$&$0.0004$&$0.0004$&$0.0004$&$0.0004$&$0.0004$
  \end{tabular}
\end{table}

\begin{table}[t]
  \caption{Error on randomly generated data.}
  \label{tab:random}
  \centering\begin{tabular}{|ccc|}
    \toprule
    $n$ & $d$ & $E(s,\mathcal{D})$ \\
    \midrule
    $10^{6}$ & $10$ & $2.55*10^{-4} \pm 1.13 * 10^{-8}$ \\
    $10^{6}$ & $50$ & $2.55*10^{-4} \pm 5.36 * 10^{-9}$ \\
    $10^{6}$ & $250$ & $2.55*10^{-4} \pm 2.31 * 10^{-9}$ \\
    \midrule
    $10^{7}$ & $10$ & $3.08*10^{-4} \pm 4.78*10^{-10}$\\
    $10^{7}$ & $50$ & $3.08*10^{-4} \pm 6.57*10^{-10}$\\
    $10^{7}$ & $250$ & $3.08*10^{-4} \pm 6.16*10^{-10}$\\
    \midrule
    $10^{8}$ & $10$ & $3.55*10^{-4} \pm 3.67*10^{-11}$\\
    $10^{8}$ & $50$ & $3.55*10^{-4} \pm 1.34*10^{-11}$\\
    $10^{8}$ & $250$ & $3.55*10^{-4} \pm 2.69*10^{-11}$\\
    \bottomrule
  \end{tabular}
\end{table}

To further understand how our approximation procedure behaves we conducted
experiments on random data. We considered different data sizes and different
amount of attributes. For this, we experimented with real-valued datasets, i.e.
datasets represented by an attribute matrix $\mathcal{X} \in \mathbb{R}^{n
\times d}$. Here, the feature functions are given by the data columns. To be
more detailed, the considered geometric dataset is $\mathcal{D}=(\{
\mathcal{X}_{i}\mid i \in \{1,\dots, n \}\}, \{X_{i} \mapsto X_{i,j} \mid j \in
\{1,\dots, d\}\}, \nu)$. Here, $\nu$ is again the normalized counting measure
and $\mathcal{X}_i$ is he $i-th$ row vector of $\mathcal{X}$. We iterate $n$
through $\{10^{6}, 10^{7}, 10^{8}\}$ and $d$ through $\{10, 50, 250\}$. We
repeat all experiments $3$ times. For all datasets, we build a support sequence
as described in~\cref{sec:sups} with $l=100,000$. The results can be found
in~\cref{tab:random}. 

For all datasets, the errors are small and the accuracy is over $99.9\%$ for all considered data sizes. The difference in the error for different values
of $d$ is negligible. Furthermore, we have small standard deviations. All this
indicates that $l=100,000$ is a reasonable default choice that leads to sufficient
approximations in a large range of data and attribute sizes.

\section{Conclusion and Future Work}
\label{sec:conc}
We presented a principle way to efficiently compute the intrinsic dimension (ID) of
geometric datasets. Our approach is based on an axiomatic foundation
and accounts for underlying
structures and is therefore especially tailored to the
field of geometric learning. We proposed a novel speed up technique
for an algorithm which has quadratic complexity with respect to the
amount of data points. This enabled us to compute the ID of several real-world graphs with up to millions
of nodes. Equipped with this ability, we shed light on connections of
classification performances of graph neural networks and the observed
intrinsic dimension for common benchmark datasets.
 Finally, using a novel
approximation technique, we were able to show that our method scales
to graphs with over 100 million nodes and billions of edges. We
illustrated this by using the well-known \textbf{ogbn-papers100M}
dataset.

Future work includes the identification of suitable feature functions
for other domains, such as learning on text or image
data. Incorporating the structure of such datasets into the
computation of intrinsic dimensionality is an open research problem.
Another promising research direction is to
investigate how the ID of datasets could be manipulated. Since our investigations suggest connections
between a low ID and high classification performances, this has the
potential to enhance
learning procedures.

\subsubsection*{Acknowledgement}
The authors thank the State of Hesse, Germany for funding this work as part of
the LOEWE Exploration project ``Dimension Curse Detector" under grant
LOEWE/5/A002/519/06.00.003(0007)/E19.

\bibliography{literature}
\bibliographystyle{tmlr}

\newpage
\appendix
\section{Appendix}

\subsection{Setup of SIGN classifiers}
\label{sec:setup}
For \textbf{PubMed}, \textbf{Cora} and \textbf{CiteSeer}, we train on
the classification task provided by Pytorch Geometric~\citep{fey19}
which was
earlier studied by~\citet{yang16}. All Open Graph Benchmark datasets
are trained and tested on the official \emph{node property prediction}
task.\footnote{\url{https://ogb.stanford.edu/docs/nodeprop/}}  Our
goal is not to find optimal classifiers but to discover connections
between the choice of $k$, the ID and classifier performance. Thus, we
omit excessive parameter tuning and stick to reasonable standard
parameters. For all tasks, we use a simple SIGN model~\cite{rossi20} with one hidden
inception layer and one classification layer. For \textbf{PubMed},
\textbf{CiteSeer} and \textbf{Cora}, we use batch sizes of 256, hidden
layer size of 64 and dropout at the input and hidden layer with
$0.5$. The learning rate is set to $0.01$. All these parameters were
taken from \citet{kipf17}. For \textbf{ogbn-arxiv}, \textbf{ogbn-mag}
and \textbf{ogbn-products}, we stick to the parameters from the SIGN
implementations on the OGB leaderbord. For \textbf{ogbn-arxiv}, we use
a hidden dimension of $512$, dropout at the input with $0.1$ and with
$0.5$ at the hidden layer.  For \textbf{ogbn-mag}, we use a hidden
dimension of $512$, do not dropout at the input and use dropout
with $0.5$ at the hidden layer. For \textbf{ogbn-products}, we use a hidden dimension of
$512$, input dropout of $0.3$ and hidden layer dropout of $0.4$. For
all ogbn tasks, the
learning rate is $0.001$ and the batch-size $50000$.
For all experiments, we train for a maximum of $1000$ epochs with early stopping on the
validation accuracy. Here, we use a patience of $15$. These are the
standard parameters of Pytorch
Lightning.\footnote{\url{https://www.pytorchlightning.ai/}} For all
models, we use an Adam optimizer with weight decay of $0.0001$. We
report mean test accuracies over 10 runs. The intrinsic dimensions and the test accuracy are
shown in~\cref{tab:k_hop_ids}.

\subsection{Details on Baseline ID Estimator}
\label{sec:impl_baseline}

To use the MLE ID, we have to convert the $k$-hop geometric
dataset $(V,F_{D,k},\nu)$ of graph data $D=(\mathcal{X},G)$, where $\mathcal{X} \in
\mathbb{R}^{n \times d}$ into a
real-valued feature matrix $\hat{X}$. This done by concatenating the rows
of $\mathcal{X}$ with the rows of $\hat{A}\mathcal{X},\dots \hat{A}^k\mathcal{X}$, i.e., $\hat{X} \in
\mathbb{R}^{n \times (k+1)d}$ with
\begin{equation*}
  \hat{X}_{i,j} \coloneqq
  \begin{cases}
    \mathcal{X}_{i,j} & j \in \{1,\dots,d\}, \\
    (A^n\mathcal{X})_{i, \hat{j}} & j = nd+\hat{j} \texttt{ for } n \in \{1, \dots, k\},
    \hat{j} \in \{1,\dots d\}.
  \end{cases}
\end{equation*}

The MLE is given via
\begin{equation}\label{eq:mle}
  \mle(\hat{X}) \coloneqq \frac{1}{n(k-1)}\sum_{i=1}^n\sum_{j=1}^{l-1}\log(\frac{d(\hat{X}_i,N_l(\hat{X}_i))}{d(\hat{X}_i,N_j(\hat{X}_i))}),
\end{equation}
where $d$ is the euclidean metric and $N_j(\hat{X}_i)$ is the $j$-th nearest
neighbor of $\hat{X}_i$ with respect to the Euclidean metric. Thus, the MLE
depends on a parameter $l$, which we set to $5$.

We implement the MLE by using the \emph{NearestNeighbors}
class of scikit-learn~\citep{pedregosa11} and then building the mean
of all $\log(\frac{d(\hat{X}_i,N_5(\hat{X}_i))}{d(\hat{X}_i,N_j(\hat{X}_i))})$ with $i \in
\{1,\dots,n\}$ and $j \in \{1,\dots, 5\}$. Here, we skip all elements
where $d(\hat{X}_i,N_j(\hat{X}_i))=0$. This can happen, when $\hat{X}$ has duplicated
rows, representing data points with equal attribute vectors.

For \textbf{ogbn-mag} and \textbf{ogbn-products},
computing~\cref{eq:mle} is not possible due to performance
reasons. Here, we sample $169,343$ indices\footnote{This is the amount of
nodes of ogbn-arxiv, the largest network for which the full
computation was feasible.} $I \subset \{1,\dots, n\}$ and
only compute
\begin{equation*}
  \mle(\hat{X}) \coloneqq \frac{1}{n(k-1)}\sum_{i \in I}\sum_{j=1}^{l-1}\log(\frac{d(\hat{X}_i,N_l(\hat{X}_i))}{d(\hat{X}_i,N_j(\hat{X}_i))}).
\end{equation*}.
\end{document}